\documentclass[10pt,journal,compsoc]{IEEEtran}

\usepackage{hyperref}
\usepackage{courier}
\usepackage{epsfig}
\usepackage{graphicx}
\usepackage{amsmath}
\usepackage{mathtools}
\usepackage{amsthm}
\usepackage{amssymb}
\usepackage{amscd}
\usepackage{times}
\usepackage{subfigure} 
\usepackage{color}
\usepackage{enumerate}
\usepackage{epstopdf}

\newtheorem{thm}{Theorem}
\newtheorem{lem}{Lemma}
\newtheorem{prp}{Proposition}
\newtheorem{dfn}{Definition}

\def\0{\boldsymbol{0}}
\def\1{\boldsymbol{1}}
\def\b{\boldsymbol{b}}
\def\c{\boldsymbol{c}}
\def\e{\boldsymbol{e}}

\def\x{\boldsymbol{x}}
\def\y{\boldsymbol{y}}
\def\v{\boldsymbol{v}}

\def\bxi{\boldsymbol{\xi}}

\def\bomega{\boldsymbol{\omega}}

\def\w{\boldsymbol{w}}

\def \E{\mathcal{E}}

\def\bP{\boldsymbol{P}}

\def\cS{\mathcal{S}}

\def\V{\mathcal{V}}

\def\bY{\boldsymbol{Y}}

\def\bW{\boldsymbol{W}}

\def\bX{\boldsymbol{X}}
\def\transpose{\top} 

\DeclareMathOperator*{\Span}{Span}

\DeclareMathOperator*{\diag}{diag}

\renewcommand{\Re}{{\mathbb{R}}}

\DeclarePairedDelimiter\norm{\lVert}{\rVert}
\DeclarePairedDelimiter\abs{\lvert}{\rvert}

\title{\LARGE \bf Theoretical Analysis of Sparse Subspace Clustering \\ with Missing Entries}

\author{Manolis C. Tsakiris \, \, \,  \, \, \, Ren\'e Vidal
\thanks{Manolis C. Tsakiris is with the School of Information Science and Technology, ShanghaiTech University, Shanghai, China. email: mtsakiris@shanghaitech.edu.cn} 

\thanks{Ren\'e Vidal is with the Department of Biomedical Engineering, Johns Hopkins University, Baltimore, USA. email: rvidal@cis.jhu.edu}}

\begin{document}

\maketitle

\begin{abstract}
Sparse Subspace Clustering (SSC) is a popular unsupervised machine learning
method for clustering data lying close to an unknown union of low-dimensional linear 
subspaces; a problem with numerous applications in pattern recognition and
computer vision. Even though the behavior of SSC for complete data is by now well-understood, little is known about its theoretical properties when applied to data with missing entries. In this paper we give theoretical guarantees for SSC with incomplete data, and analytically establish that projecting the zero-filled data onto the observation pattern of the point being expressed leads to a substantial improvement in performance. The main insight that stems from our analysis is that even though the projection induces additional missing entries, this is counterbalanced by the fact that the projected and zero-filled data are in effect incomplete points associated with the union of the corresponding projected subspaces, with respect to which the point being expressed is complete. The significance of this phenomenon potentially extends to the entire class of self-expressive methods. 
\end{abstract}


\section{INTRODUCTION} \label{section:Introduction}

\IEEEPARstart{C}{lustering} data lying close to an unknown union of low-dimensional linear subspaces is a fundamental problem in unsupervised machine learning, known as \emph{Subspace Clustering} or \emph{Generalized Principal Component Analysis} \cite{Vidal:GPCAbook}. Indeed, this problem is intimately related to the extension of the classical Principal Component Analysis (PCA) to multiple subspaces, and in recent years has found numerous applications in machine learning, computer vision, pattern recognition, bioinformatics and systems theory. Moreover, recent work is beginning to explore connections between subspace clustering and deep learning, with the goal of learning unions of low-dimensional non-linear manifolds \cite{Peng:IJCAI16}. 
 
Among a variety of subspace clustering methods \cite{Vidal:GPCAbook} including algebraic \cite{Vidal:PAMI05,Tsakiris:SIAM17,Tsakiris:AffinePAMI17}, iterative \cite{Bradley:JGO00}, recursive \cite{RANSAC,Tsakiris:ICML17}, and spectral \cite{Sekmen:ACHA17,Heckel:TIT15,Lu:ECCV12,Chen:IJCV09} techniques, Sparse Subspace Clustering (SSC) \cite{Elhamifar:CVPR09, Elhamifar:TPAMI13} is one of the most popular methods. The reason is that it exhibits a very competitive performance in real-world datasets, it admits efficient numerical implementations, and is supported by a rich body of theory \cite{Elhamifar:TPAMI13,Soltanolkotabi:AS12,Wang:JMLR16,Soltanolkotabi:AS14}. In addition, SSC is able to cluster data from incomplete observations reasonably well \cite{Yang:ICML15}, which is an important problem    \cite{Ongie:ICML17,Alarcon:ICML16,Elhamifar:NIPS16,Yang:ICML15,Heckel:TIT15,Alarcon-ISIT15,Eriksson:AISTATS12,Recht:JMLR11,Balzano:ISIT10}, since in many applications not all features are available for every data point: Users of recommendation systems only rate a few items, medical patients undergo only a few tests and treatments, dynamic processes are observed across short time intervals and so on.

Even though the theoretical foundations of SSC are by now mature, there are many lingering open questions. For example, it is still unclear whether weaker conditions exist for the performance of SSC even for uncorrupted data; contrast this to the recent study \cite{You:ICML15}, which establishes a hierarchy of such conditions for sparse subspace recovery. More importantly, even though a satisfactory theory for SSC with general noise does exist \cite{Wang:JMLR16}, the theoretical properties of SSC for data with missing entries remain elusive, with \cite{Wang:ISIT16,Charles:arXiv18} being recent efforts towards that end. In particular, the conditions of \cite{Wang:ISIT16} are hard to interpret, and in addition they refer to the formulation of SSC with exact self-expressiveness equality constraint; not an optimal choice for corrupted data. On the other hand, through a different methodology the work of \cite{Charles:arXiv18} provides bounds similar to a subset of the results here\footnote{We note that an earlier version of \cite{Charles:arXiv18}, namely \cite{Charles:arXiv17}, studied an $\ell_1$ formulation of SSC. We pointed out to the authors that the arguments between (A.11) and (A.12) were incorrect, which prompted them to the study of a Lasso version of SSC with missing entries, what we call in this paper ZF-SSC.}. 

In this paper we provide a novel theoretical analysis of SSC for incomplete data. More precisely, we provide theoretical performance guarantees for SSC applied to i) \emph{Zero-Filled} data (ZF-SSC), in which case all unobserved entries are filled with zeros, and ii) \emph{Projected-Zero-Filled} data (PZF-SSC), in which case all unobserved entries are filled with zeros and in addition all data points are projected onto the observation pattern of the point being expressed each time\footnote{This is called \emph{EWZF-SSC} in \cite{Yang:ICML15}; here we have taken the liberty to rename the method according to the more suggestive name PZF-SSC.}. A direct comparison of the tolerable bounds of missing entries for ZF-SSC (Theorem \ref{thm:ZF-probabilistic}) and PZF-SSC (Theorem \ref{thm:PZF-probabilistic}) serves as a theoretical justification for the former being a better method than the latter, which is in agreement with existing experimental evidence \cite{Yang:ICML15}. Since PZF data have in principle more missing entries than ZF data, this is a remarkable phenomenon, of potentially wider significance to the entire class of self-expressive-based methods, e.g., \cite{Liu:TPAMI13,Lu:ECCV12,Elhamifar:TPAMI13,Wang:NIPS13-LRR+SSC,You:CVPR16-EnSC}. 

Our analysis has the distinctive feature (say as compared to \cite{Wang:JMLR16}) of not decoupling the noise from the data until the very end, and also of making use of a novel observation (Lemma \ref{lem:New-vUB}) about the Euclidean norm of the so-called \emph{dual vector}. This, is not only a convenient tool for the analysis of SSC with incomplete data, but when applied to data with generic noise, it leads to higher tolerable levels of noise (Theorem \ref{thm:SSCgenericNoise}) than those in \cite{Wang:JMLR16}, and when applied to uncorrupted data, it leads to an even weaker and more easily computable condition (Theorem \ref{thm:SSCuncorrupted}) than that in \cite{Soltanolkotabi:AS12}.

The rest of the paper is organized as follows. In \S \ref{subsection:Notation} we introduce the notation and main objects of this paper. In \S \ref{section:SSCreview} we review SSC for uncorrupted data, and discuss the two known elementary formulations of SSC for incomplete data, i.e., ZF-SSC and PZF-SSC. In \S \ref{section:SSCincomplete} we present the main contributions of this paper, which consist of deterministic and probabilistic characterizations of the tolerable percentage of missing entries for ZF-SSC and PZF-SSC, as well as a formal comparison between the two methods. 
In \S \ref{section:OtherImplications} we present new results for SSC with uncorrupted data or with data corrupted by generic noise. All proofs are organized in \S \ref{section:Proofs}, while the appendix contains necessary technical facts used throughout the analysis.

\subsection{Notation and Main Objects} \label{subsection:Notation}

The nature of the problem studied in this paper calls for a rather heavy notation, which we have strived to simplify and unify as much as possible. To avoid introducing complicated notation amidst other technical developments, we have found it convenient to gather all relevant objects in Definition \ref{dfn:BasicObjects}, which the reader is encouraged to refer to when necessary. Other than that, for $\ell$ a positive integer, we define 
$[\ell]:=\{1,\dots,\ell\}$. For a vector $\w \in \Re^D$ we define $\hat{\w} := \w / \norm{\w}_2$, if $\w \neq \0$, and $\hat{\w} := \0$, otherwise. For $\V$ any linear subspace of $\Re^D$, we denote by $\bP_{\V}$ the square matrix that represents the orthogonal projection of $\Re^D$ onto $\V$. Given a binary relation, RHS stands for \emph{Right-Hand-Side}, and similarly for LHS. Finally, $\langle \cdot,\cdot \rangle$ is the standard inner product.

\begin{dfn}\label{dfn:BasicObjects} We define the following objects:  
\begin{enumerate}
\item {\bf The linear subspaces:} For $i \in [n]$, we let $\cS_i$ be a linear subspace of $\Re^D$, where $\dim \cS_i = d_i<D$.

\item {\bf The complete data:} With an abuse of notation, we let 
\begin{align}
\bX = [\bX^{(1)},\dots,\bX^{(n)}] \boldsymbol{\Gamma} \in \Re^{D \times N}
\end{align} denote a data matrix as well as a set (formed by the columns of this matrix) of unit $\ell_2$-norm points in the union of the linear subspaces $\cS_i, \, i \in [n]$, where $\bX^{(i)} = [\x_1^{(i)}, \dots, \x_{N_i}^{(i)}] \subset \cS_i$, $\Span(\bX^{(i)}) = \cS_i$, and $\boldsymbol{\Gamma}$ is an unknown permutation, indicating that the clustering of the points with respect to the subspaces is unknown. We define $\bX_{-1}^{(1)} := \bX^{(1)} \setminus \{ \x_1^{(1)}\}$, $\bX_{-1} := \bX \setminus \{ \x_1^{(1)}\}$, and $\bX^{(-1)} := \bX \setminus \bX^{(1)}$, where $\setminus$ denotes set-theoretic difference.

\item {\bf The pattern of missing entries:} For every point $\x_j^{(i)} \in \Re^D$ we consider an observation pattern $\bomega_{j}^{(i)} \in \{0,1 \}^{D}$, where a value of $1$ indicates an observed entry, while a value of $0$ indicates an unobserved entry. We assume each $\bomega_{j}^{(i)}$ has precisely $m$ zeros.  We let $\tilde{\bomega}_{j}^{(i)} := \1 - \bomega_{j}^{(i)}$, where $\1$ is the vector of all ones.

\item {\bf The observed/unobserved coordinate subspaces:} We let $\bar{\E}_j^{(i)} := \Span \{\e_k: \,  \e_k^\transpose \bomega_j^{(i)} \neq 0 \}$, with $\e_k$ the canonical vector of $\Re^D$ with zeros everywhere and a $1$ at position $k$. The orthogonal projection onto $\bar{\E}_j^{(i)}$ is given by 
$\bar{\bP}_j^{(i)}:= \diag(\bomega_j^{(i)})$, the matrix with $\bomega_j^{(i)}$ on its diagonal and zeros everywhere else. $\tilde{\E}_j^{(i)}$ is the orthogonal complement of $\bar{\E}_j^{(i)}$, and $\tilde{\bP}_j^{(i)} = \diag(\tilde{\bomega}_j^{(i)})$ is the orthogonal projection onto $\tilde{\E}_j^{(i)}$. 

\item {\bf The zero-filled data (ZF-data):} We let $\bar{\bX} \in \Re^{D \times N}$ be the data $\bX$ with zeros appearing in the unobserved entries, i.e., the column of $\bar{\bX}$ associated to point $\x_j^{(i)}$ is $\bar{\x}_j^{(i)}:=\bar{\bP}_j^{(i)}\x_j^{(i)}, \, \forall i,j$.

\item {\bf The projected data:} We let $\dot{\bX}:= \bar{\bP}_1^{(1)} \bX$ be the projection of the data $\bX$ onto the observed coordinate subspace $\bar{\E}_1^{(1)}$ associated to point $\x_1^{(1)}$. The column of $\dot{\bX}$ associated to $\x_j^{(i)}$ is $\dot{\x}_j^{(i)}:=\bar{\bP}_1^{(1)}\x_j^{(i)}, \, \forall i,j$. 

\item {\bf The projected and zero-filled data (PZF-data):} We let $\dot{\bar{\bX}}$ be the projection of the zero-filled data onto $\bar{\E}_1^{(1)}$, i.e., $\dot{\bar{\bX}}:=\bar{\bP}_1^{(1)} \bar{\bX}$. The column of $\dot{\bar{\bX}}$ associated to point $\x_j^{(i)}$ is $\dot{\bar{\x}}_j^{(i)}:=\bar{\bP}_1^{(1)}\bar{\x}_j^{(i)}, \, \forall i,j$. 

\item {\bf The unobserved data:} We define $\tilde{\bX}$ to be the unobserved components of the data, i.e., $\tilde{\bX}:=\bX - \bar{\bX}$, and $\tilde{\x}_j^{(i)}:=\tilde{\bP}_1^{(1)}\x_j^{(i)}, \, \forall i,j$. Similarly, for PZF data we define $\dot{\tilde{\bX}}:=\dot{\bX} - \dot{\bar{\bX}}$, and $\dot{\tilde{\x}}_j^{(i)}:=\tilde{\bP}_1^{(1)}\bar{\x}_j^{(i)}, \, \forall i,j$.

\item {\bf The projected subspaces:} For $i \in [n]$, we let $\dot{\cS}_i \subset \Re^D$ be the orthogonal projection of $\cS_i$ onto the subspace $\bar{\mathcal{E}}_1^{(1)}$. In other words, if $\b_1^{(i)},\dots,\b_{d_i}^{(i)}$ is a basis for $\cS_i$, then $\dot{\cS}_i$ is the subspace of $\Re^D$ spanned by the vectors $ \bar{\bP}_1^{(1)} \b_k^{(i)}, \forall k \in [d_i]$. 

\item {\bf The inradius:} We let $r$ be the relative inradius of the symmetrized convex hull $\mathcal{Q}$ of all points $\bX_{-1}^{(1)}$ lying in subspace $\cS_1$, except point $\x_1^{(1)}$, i.e., $r$ is the radius of the largest Euclidean ball of $\cS_1$ contained in $\mathcal{Q}$.

\item {\bf The dual directions:} For $\bW=\bX, \bar{\bX}, \dot{\bar{\bX}}$ corresponding to complete data $\bX$, ZF-data $\bar{\bX}$ and PZF-data $\dot{\bar{\bX}}$, consider the reduced Lasso-SSC problem
\small
\begin{align}
\min_{\c,\e} \, \left\|\c \right\|_1 + \frac{\lambda}{2} \left\| \e\right\|_2^2 \, \, \, \text{s.t.} \, \, \
\w_1^{(1)} = \bW_{-1}^{(1)} \c + \e, \label{eq:dfn-LassoReduced}
\end{align} \normalsize corresponding to either complete data $\bX$, ZF-data $\bar{\bX}$ or PZF-data $\dot{\bar{\bX}}$. Consider the dual problem 
\small
\begin{align}
\max_{\v} \,  \langle \v, \w_1^{(1)}\rangle - \frac{1}{2 \lambda} \left\|\v \right\|_2^2 \, \, \, \text{s.t.} \, \, \, \norm{\v^\top \bW_{-1}^{(1)} }_{\infty} \le1. \label{eq:dfn-LassDualReducedZF}
\end{align} \normalsize Let $\v^*_{\lambda}, \bar{\v}^*_{\lambda}, \dot{\bar{\v}}^*_{\lambda}$ be the optimal solution to problem \eqref{eq:dfn-LassDualReducedZF} corresponding to $\bW = \bX, \bar{\bX},\dot{\bar{\bX}}$ respectively; these solutions are unique because \eqref{eq:dfn-LassDualReducedZF} is strongly convex. Then we define the corresponding dual directions $\hat{\v}_{1,\lambda},\hat{\bar{\v}}_{1,\lambda},\hat{\dot{\bar{\v}}}_{1,\lambda}$ to be the normalized projections of $\v^*_{\lambda}, \bar{\v}^*_{\lambda}, \dot{\bar{\v}}^*_{\lambda}$ onto $\cS_1,\cS_1,\dot{\cS}_1$ respectively (if any of these projections is equal to zero, we define the corresponding dual direction to be the zero vector). 

\item {\bf The inter-subspace coherences:} We define the inter-subspace coherences for complete data, ZF-data, and PZF-data respectively as
\begin{align}
\mu_{\lambda} &:= \max_{i>1, \, k \in [N_i]} \abs{ \langle \x_k^{(i)},\hat{\v}_{1,\lambda} \rangle } \\
\bar{\mu}_{\lambda} &:= \max_{i>1, \, k \in [N_i]} \abs{ \langle \bar{\x}_k^{(i)},\hat{\bar{\v}}_{1,\lambda} \rangle } \\
\dot{\bar{\mu}}_{\lambda} &:= \max_{i>1, \, k \in [N_i]} \abs{ \langle \dot{\bar{\x}}_k^{(i)},\hat{\dot{\bar{\v}}}_{1,\lambda} \rangle }. \label{eq:PZF-mu}
\end{align}

\item {\bf The intra-subspace coherences:}
\begin{align}
\zeta &:= \norm{(\bX_{-1}^{(1)})^\transpose \x_1^{(1)}}_{\infty}, \label{eq:zeta}\\
\bar{\zeta} &:= \norm{(\bar{\bX}_{-1}^{(1)})^\transpose \bar{\x}_1^{(1)}}_{\infty}, \label{eq:bar-zeta}\\
\dot{\bar{\zeta}}&:= \norm{(\dot{\bar{\bX}}_{-1}^{(1)})^\transpose \dot{\bar{\x}}_1^{(1)}}_{\infty}, \, \, \,  (\bar{\zeta}=\dot{\bar{\zeta}}) 
\end{align}

\item {\bf Other quantities:}
\begin{align}
\bar{\eta} &:=  \norm{\bar{\x}_1^{(1)}}_2, \label{eq:eta}\\
\dot{\bar{\eta}} &:= \norm{\dot{\bar{\x}}_1^{(1)}}_2,\, \, \, (\bar{\eta}=\dot{\bar{\eta}}) \label{eq:dot-eta}\\
\bar{\gamma}&:= \max_{\substack{i>1, k \in [N_i], \\j \in [N_1]}} 
\abs{\langle \bar{\x}_k^{(i)}, \bP_{\cS_1^\perp} \tilde{\x}_j^{(1)}  \rangle}  \label{eq:gamma-ZF}\\
\dot{\bar{\gamma}} &:= \max_{\substack{i>1, k \in [N_i], \\j \in [N_1]}}  
\abs{\langle \dot{\bar{\x}}_k^{(i)}, \bP_{\dot{\cS}_1^\perp} \dot{\tilde{\x}}_j^{(1)}  \rangle}. 
\end{align}
\end{enumerate}
\end{dfn} 

\section{Review of Sparse Subspace Clustering} \label{section:SSCreview}

We begin by reviewing Sparse Subspace Clustering (SSC) for data with no corruptions (\S \ref{subsection:SSCreviewUncorrupted}), as well as the two fundamental approaches to SSC for incomplete data (\S \ref{subsection:SSCreviewMissing}), which this paper is devoted to analyzing.

\subsection{SSC With Uncorrupted Data} \label{subsection:SSCreviewUncorrupted}

In the absence of data corruptions (noise, missing entries, outliers, etc.) we consider a data matrix $\bX \in \Re^{N\times D}$ as in Definition \ref{dfn:BasicObjects}, whose columns are unit-$\ell_2$ points\footnote{This is not an essential assumption, rather it is convenient for the purpose of theoretical analysis.} that lie in an unknown union of low-dimensional linear subspaces $\bigcup_{i=1}^n \cS_i \subset \Re^D$, with $d_i := \dim (\cS_i)$. Thus $\bX=[\bX^{(1)} \cdots \bX^{(n)}] \boldsymbol{\Gamma}$, where each $\bX^{(i)}:=[\x_1^{(i)} \cdots \x_{N_i}^{(i)} ] \in \Re^{N_i \times D}$ consists of $N_i$ points spanning subspace $\cS_i$, and $\boldsymbol{\Gamma}$ is an unknown permutation, indicating that the clustering of the points is unknown. 

Among a variety of methods \cite{Vidal:GPCAbook} for retrieving the clusters $\{\bX^{(i)}\}$, one may apply Sparse Subspace Clustering (SSC) \cite{Elhamifar:CVPR09, Elhamifar:TPAMI13}, whose main principle is to express each point in $\bX$ as a sparse linear combination of other points in $\bX$. Specifically, we seek an expression, say, of point $\x_1^{(1)}$ as a sparse linear combination of all other points $\bX_{-1} := \bX \setminus \{\x_1^{(1)} \}$ by means of the basis pursuit problem \cite{Chen:SIAM98} \begin{align}
\min_{\c \in \Re^{N_1-1}} \, \, \, \left\| \c \right\|_1 \, \, \, \text{s.t.} \, \, \, \x_1^{(1)}  = \bX_{-1} \c, \label{eq:SSC-BP}
\end{align} and then form an affinity graph in which we connect $\x_1^{(1)}$ to those points of $\bX_{-1}$ that correspond to the support (non-zero coefficients) of the computed optimal solution of \eqref{eq:SSC-BP}. Clearly, we want these points to lie in the same subspace as $\x_1^{(1)}$, i.e., to be points of $\bX_{-1}^{(1)} := \bX^{(1)} \setminus \{\x_1^{(1)} \}$, in which case we say that 
the solution is \emph{subspace preserving}. When this is true for the expression of each and every point in $\bX$, then the corresponding affinity graph contains no connections between points in different subspaces, i.e., it is a subspace preserving graph. Assuming that points within each subspace are sufficiently well connected, the affinity graph will have precisely $n$ connected components, and spectral clustering is guaranteed to furnish the correct clusters.

Often, it is more practical to search for approximate sparse linear combinations rather exact ones as in \eqref{eq:SSC-BP}. Thus one may approximately express point $\x_1^{(1)}$ by solving the Lasso problem \cite{Tibshirani:EJS2013} 
\begin{align}
\min_{\c} \, \, \, \left\|\c \right\|_1 + \frac{\lambda}{2} \left\| \e\right\|_2^2 \, \, \, \text{s.t.} \, \, \, \x_1^{(1)} = \bX_{-1} \c + \e, \label{eq:SSC-Lasso}
\end{align} where $\e$ represents the reconstruction error. We have the following known guarantee\footnote{For ease of notation, we state all results in terms of expressing $\x_1^{(1)}$.}:

\begin{thm}[SSC with uncorrupted data, deterministic \cite{Wang:JMLR16}] \label{thm:SSCuncorruptedOld} Recall the notation of Definition \ref{dfn:BasicObjects}, and suppose that 
\begin{align}
\mu_{\lambda} < r \, \, \, \text{and} \, \, \, 1/\zeta< \lambda. \label{eq:OldUncorrupted}
\end{align} Then every optimal solution to the Lasso SSC problem \eqref{eq:SSC-Lasso} is non-zero and subspace preserving. 
\end{thm}

\noindent Theorem \ref{thm:SSCuncorruptedOld} can be interpreted as follows: If all data points from $\cS_1$ other than $\x_{1}^{(1)}$ are well distributed (large $r$), the data points from other subspaces are sufficiently far from $\cS_1$ as measured by their inner product with the dual direction $\hat{\v}_{1,\lambda}$ (small $\mu_{\lambda}$), and the reconstruction error is penalized sufficiently enough (large $\lambda$), then the Lasso problem \eqref{eq:SSC-Lasso} is guaranteed to furnish non-zero and subspace preserving solutions.

Employing Lemmas  \ref{lem:r} and \ref{lem:inner} in the Appendix, Theorem \ref{thm:SSCuncorruptedOld} can be converted to a probabilistic statement under a simplified fully random model.

\begin{dfn}[Random model] \label{dfn:RandomModel}
For each $i \in [n]$, let the $i$th subspace be chosen uniformly at random from the Grassmannian manifold of $d$-dimensional subspaces of $\Re^D$. Moreover, let\footnote{For simplicity, we assume that $n$ divides $N$.} $N/n=:\rho d+1$ points be chosen uniformly at random from the intersection of each subspace and the unit sphere $\mathbb{S}^{D-1}$. Finally, define the quantities 
\begin{align}
\alpha:= \sqrt{\frac{\log(\rho)}{16d}}, \, \, \, \beta:= \sqrt{\frac{6\log(N)}{D}}. \label{eq:alpha-beta}
\end{align}
\end{dfn}

\begin{thm}[SSC with uncorrupted data, probabilistic \cite{Wang:JMLR16}, \cite{Soltanolkotabi:AS12}] \label{thm:SSC-probabilistic}
Consider the random model of Definition \ref{dfn:RandomModel}. If $\rho$ is larger than a constant, $\lambda> 1/\alpha$, and 
\begin{align}
\alpha > \beta \label{eq:SSC-probabilistic},
\end{align} then any optimal solution to the Lasso SSC problem \eqref{eq:SSC-Lasso} is non-zero and subspace preserving, with probability at least $1-2/N^2-\exp(-\sqrt{\rho} d)$. 
\end{thm} \noindent Condition \eqref{eq:SSC-probabilistic} agrees with intuition, since it effectively says that the subspace preserving property is easier to achieve for small relative subspace dimensions $d/D$, fewer subspaces, and more points per subspace. In \S \ref{section:SSCincomplete} we will give analogues of Theorems \ref{thm:SSCuncorruptedOld} and \ref{thm:SSC-probabilistic} for two fundamental variants of SSC for incomplete data, described next.

\subsection{SSC With Missing Entries (ZF-SSC,PZF-SSC)} \label{subsection:SSCreviewMissing}

When the data are incomplete but otherwise uncorrupted, one may consider using a low-rank matrix completion algorithm to first complete the data and then applying SSC to the completed data. However, this procedure is guaranteed to succeed only when the underlying complete matrix $\bX$ is of low rank, an assumption which might become invalid in the presence of data from many distinct subspaces. As a simple alternative, one may fill with zeros the unobserved entries to obtain a zero-filled data matrix $\bar{\bX}$ exactly as in Definition \ref{dfn:BasicObjects}, and subsequently solve the problem
\begin{align}
\min_{\c} \, \, \, \left\|\c \right\|_1 + \frac{\lambda}{2} \left\| \e\right\|_2^2 \, \, \, \text{s.t.} \, \, \
\bar{\x}_1^{(1)} = \bar{\bX}_{-1} \c + \e. \label{eq:Lasso-ZF}
\end{align} This approach is called \emph{Zero-Filled SSC} (ZF-SSC) \cite{Yang:ICML15}.  
In spite of its simplicity, ZF-SSC is known to have a reasonable performance, which can be further considerably improved as compared to the results reported in \cite{Yang:ICML15}, simply by setting the parameter $\lambda$ in the optimization problem \eqref{eq:Lasso-ZF} adaptively for each point \footnote{Specifically, when expressing point $\bar{\x}_1^{(1)}$ set the parameter $\lambda = a/ \norm{\bar{\bX}_{-1}^\transpose \bar{\x}_1^{(1)}}_\infty$ and similarly for other points, where $a>1$ is a parameter chosen the same for all points.}.

Even so, ZF-SSC has an evident shortcoming: it penalizes the reconstruction error of the zero vector along the unobserved part of the point being expressed, which is clearly an undesirable feature of the method. More precisely, letting $\bar{\mathcal{E}}_1^{(1)}$ and $\tilde{\mathcal{E}}_1^{(1)}$ be, respectively, the observed and unobserved subspaces associated to point ${\x}_1^{(1)}$, and $\bar{\bP}_1^{(1)}$ and $\tilde{\bP}_1^{(1)}$ the orthogonal projections onto them (see Definition \ref{dfn:BasicObjects}), and recalling that $(\bar{\mathcal{E}}_1^{(1)})^\perp=\tilde{\mathcal{E}}_1^{(1)}$, we have that
\begin{align}
\bar{\x}_1^{(1)} &=\bar{\bP}_1^{(1)} \bar{\x}_1^{(1)}, \, \, \, \text{and} \\
\bar{\bX}_{-1} &=\bar{\bP}_1^{(1)}\bar{\bX}_{-1} + \tilde{\bP}_1^{(1)}\bar{\bX}_{-1},
\end{align} and so we can rewrite the objective function of ZF-SSC as 
\begin{align}
& \norm{\c}_1 + 
\frac{\lambda}{2} \norm{\bar{\x}_1^{(1)} -  \bar{\bX}_{-1}\c}_2^2  =  \norm{\c}_1 +\\ 
& \frac{\lambda}{2} \norm{\bar{\x}_1^{(1)} -  \bar{\bP}_1^{(1)}\bar{\bX}_{-1}\c}_2^2 + \frac{\lambda}{2}\norm{\tilde{\bP}_1^{(1)}\bar{\bX}_{-1}\c}_2^2.\end{align} 
We then see that ZF-SSC penalizes the reconstruction error $\norm{\bar{\x}_1^{(1)} -  \bar{\bP}_1^{(1)}\bar{\bX}_{-1}\c}_2$ of the observed part of $\x_1^{(1)}$, which is desirable, as well as the norm of the vector
$\tilde{\bP}_1^{(1)}\bar{\bX}_{-1}\c$. The latter is an artifact of the zero-filling process, and could bias the coefficients $\c$ away from a subspace preserving pattern. Thus, it is reasonable to remove this term and obtain self-expressive coefficients for $\bar{\x}_1^{(1)}$ by solving instead
 \begin{align}
\min_{\c} \, \, \, \norm{\c}_1 + \frac{\lambda}{2} \norm{\e}_2^2, \, \, \, \text{s.t.} \, \, \, \e = \bar{\x}_1^{(1)}-  \dot{\bar{\bX}}_{-1}\c, \label{eq:Lasso-PZF}
\end{align} where $\dot{\bar{\bX}}:=\bar{\bP}_1^{(1)}\bar{\bX}$ is the projected and zero-filled data, as in Definition \ref{dfn:BasicObjects}. \cite{Yang:ICML15} called this approach \emph{EWZF-SSC}; we here take the liberty to rename it \emph{Projected-Zero-Filled Sparse-Subspace-Clustering (PZF-SSC)}. 

PZF-SSC is known to provide accurate clustering while tolerating a higher percentage of missing entries than ZF-SSC. This is rather fascinating, since, after all, PZF-SSC works with the projected and zero-filled data $\dot{\bar{\bX}}$, which have more missing entries than the zero-filled data $\bar{\bX}$. Note that precisely because of this reason, direct application of any generic noise bound (such as that of Theorem $6$ in \cite{Wang:JMLR16}) would naively suggest that ZF-SSC tolerates more missing entries than PZF-SSC, contradicting experimental evidence. This apparent \emph{mystery} is resolved in \S \ref{section:SSCincomplete}, where we adopt a more sophisticated view of PZF-SSC, thus unveiling its superiority over ZF-SSC.

\section{SSC Theory for Incomplete Data} \label{section:SSCincomplete}

This section contains the main contributions of this paper. In \S \ref{subsection:PZF-SSC} and
\S \ref{subsection:ZF-SSC} we give deterministic and probabilistic theorems of correctness for PZF-SSC and ZF-SSC, respectively, in analogy with Theorems \ref{thm:SSCuncorruptedOld} and \ref{thm:SSC-probabilistic} for SSC with uncorrupted data. We conclude in \S \ref{subsection:Comparison} with the first theoretical comparison of the two methods.

\subsection{PZF-SSC Theory} \label{subsection:PZF-SSC}

As already remarked so far, PZF-SSC is experimentally known to be a superior method to ZF-SSC, i.e., it can provide an accurate clustering for a higher percentage of missing entries. This is remarkable, because the projected and zero-filled data $\dot{\bar{\bX}}$ (see Definition \ref{dfn:BasicObjects} for notation) that PZF-SSC operates on contain more missing entries than the zero-filled data $\bar{\bX}$ that ZF-SSC operates on. On the other hand, we already saw in \S \ref{subsection:SSCreviewMissing} that the additional zeros in $\dot{\bar{\bX}}$ are inflicted in such a way, that the objective function minimized by PZF-SSC is, at least on an intuitive level, more accurate than the one minimized by ZF-SSC. 

In this paper we give a theoretical justification for the superiority of PZF-SSC over ZF-SSC. Our main insight is the following observation: expressing point $\bar{\x}_1^{(1)} = \dot{\bar{\x}}_1^{(1)}$ as
a sparse linear combination of $\dot{\bar{\bX}}_{-1}$, can be seen as expressing the \emph{complete} point $\bar{\x}_1^{(1)}$ from partial observations $\dot{\bar{\bX}}_{-1}$ of the \emph{complete} points $\dot{\bX}_{-1}$, where now the underlying \emph{complete} data $\dot{\bX}$ lie in the union of subspaces $\bigcup_{i=1}^n \dot{\cS}_i$, i.e., the original subspaces projected onto the coordinate subspace defined by the observation pattern of the point being expressed (see Definition \ref{dfn:BasicObjects}). With this in mind, inspired by the seminal work of \cite{Wang:JMLR16}, and by 
\begin{enumerate}
\item making more frequent use of strong duality than as in the proof of Theorem $6$ in \cite{Wang:JMLR16}, 
\item using the novel Lemma \ref{lem:New-vUB} to bound the norm of the dual vector, and
\item not decoupling the \emph{noise} from the data,
\end{enumerate} we arrive at the following key result:

\begin{thm}[PZF-SSC, deterministic] \label{thm:PZF-deterministic}
With the notation of Definition \ref{dfn:BasicObjects}, further define the positive quantity
\begin{align}
\dot{\bar{\lambda}}_{\max}:= &\frac{1}{2}\Bigg\{ \frac{1}{2 \dot{\bar{\zeta}}} -\frac{\dot{\bar{\mu}}_{\lambda}}{\dot{\bar{\gamma}}\dot{\bar{\eta}}} +  \nonumber \\
& 
\Bigg(\frac{9}{4\dot{\bar{\zeta}}^2}+\frac{\dot{\bar{\mu}}_{\lambda}}{\dot{\bar{\gamma}}\dot{\bar{\eta}}\dot{\bar{\zeta}}} +\frac{2}{\dot{\bar{\gamma}}\dot{\bar{\eta}}^2} +\frac{\dot{\bar{\mu}}_{\lambda}^2}{\dot{\bar{\gamma}}^2 \dot{\bar{\eta}}^2}\Bigg)^{1/2} \Bigg\}. \label{eq:lambda-u}
\end{align} If the condition 
\begin{align}
\dot{\bar{\mu}}_{\lambda} \, \dot{\bar{\eta}} \, < \dot{\bar{\zeta}}, \label{eq:PZF-deterministic-gap}
\end{align} is true, then the open interval $\dot{\bar{\Lambda}} := (1/\dot{\bar{\zeta}}, \dot{\bar{\lambda}}_{\max})$
is non-empty. If in addition\footnote{Since the interval $\dot{\bar{\Lambda}}$ is a function of $\lambda$, it is misleading to write ``for any $\lambda \in \dot{\bar{\Lambda}}$", as the fact that $\dot{\bar{\Lambda}}$ is non-empty does not alone guarantee that also $\lambda \in \dot{\bar{\Lambda}}$. For example, this occurs in Theorem $6$ of \cite{Wang:JMLR16}.} $\lambda \in \dot{\bar{\Lambda}}$, then every optimal solution to the Lasso SSC problem \eqref{eq:Lasso-PZF} with projected and zero-filled data is non-zero and subspace preserving.
\end{thm} 

\noindent What is remarkable about Theorem \ref{thm:PZF-deterministic} is the simplicity of the condition $\dot{\bar{\mu}}_{\lambda} \, \dot{\bar{\eta}} \, < \dot{\bar{\zeta}}$, as well as its resemblance to the condition $\mu_{\lambda} < r$ of Theorem \ref{thm:SSCuncorruptedOld}. In fact, the quantity $\dot{\bar{\mu}}_{\lambda}$ is a direct analogue of the inter-subspace coherence $\mu_{\lambda}$, adjusted for the case of PZF data. Indeed, as seen from its definition in \eqref{eq:PZF-mu}, $\dot{\bar{\mu}}_{\lambda}$ is the maximum inner product between the dual direction associated to the PZF data of subspace $\cS_1$ and PZF data coming from the rest of subspaces. The quantity $\dot{\bar{\eta}} \le 1$ is the Euclidean norm of the point being expressed, which in the absence of missing entries is equal to $1$. 

Finally, to understand the quantity $\dot{\bar{\zeta}}$, we first look at its noiseless counterpart $\zeta$ defined in \eqref{eq:zeta}. This measures how well distributed are the points $\bX_{-1}^{(1)}$ with respect to point $\x_1^{(1)}$, or in other words, how coherent they are with that point. Notice here that $\zeta$ is a more relevant quantity than the inradius $r$, since the latter does not involve any information about the point being expressed. In addition, $\zeta$ is directly computable from the data, while the inradius is in principle hard to compute. Furthermore, by Proposition \ref{prp:r-charecterization} in the Appendix, it is true that $r \le \zeta$, so that if we were to replace condition $\mu_{\lambda} < r$ with condition $\mu_{\lambda} < \zeta$, we would obtain a better result. This is precisely the condition that Theorem \ref{thm:PZF-deterministic} reduces to for complete data; see Theorem \ref{thm:SSCuncorrupted}. With this interpretation in mind, we see that the quantity $\dot{\bar{\zeta}}$ captures how well distributed the PZF data $\dot{\bar{\bX}}_{-1}^{(1)}$ are with respect to the point $\bar{\x}_1^{(1)}$ that is being expressed. 

The above discussion allows interpreting condition \eqref{eq:PZF-deterministic-gap}: the PZF data $\dot{\bar{\bX}}_{-1}^{(1)}$ associated to the same subspace $\cS_1$ as the point $\dot{\bar{\x}}_1^{(1)}$ being expressed must be well distributed with respect to that point normalized (large $\dot{\bar{\zeta}}/\dot{\bar{\eta}}$), while the PZF points $\dot{\bar{\bX}}^{(-1)}$ in the rest of the subspaces must be sufficiently far away from the projected subspace $\dot{\cS}_1$, as measured by their inner product with the corresponding dual direction $\hat{\dot{\bar{\v}}}_{1,\lambda} \in \dot{\cS}_1$ (small $\dot{\bar{\mu}}_{\lambda}$). 

Next, we derive a probabilistic statement from Theorem \ref{thm:PZF-deterministic}. This is done by constructing high-probability upper and lower bounds for the LHS and RHS of \eqref{eq:PZF-deterministic-gap}, where we exploit the fact that data corruptions due to missing entries are induced by orthogonal projections, i.e., for every $\x_j^{(i)}$,
\begin{align}
\bar{\x}_j^{(i)} = \bar{\bP}_j^{(i)}\x_j^{(i)}  = \x_j^{(i)} + (-\tilde{\bP}_j^{(i)}\x_j^{(i)}).
\end{align} Using the bounds of Lemmas \ref{lem:r}, \ref{lem:inner}, and \ref{lem:RandomProjection} in the Appendix, through non-trivial calculations we obtain:

\begin{thm}[PZF-SSC, probabilistic]\label{thm:PZF-probabilistic}
Consider the random model of Definition \ref{dfn:RandomModel}. Suppose that for each point we do not observe exactly $m<D-d$ entries, with the pattern of missing entries being arbitrary, but otherwise fixed apriori. Suppose that the point density $\rho$ is larger than a constant, and let $\epsilon >0$ be a parameter that controls the probability of success. If the ratio $\omega:=m/D$ of missing entries satisfies
\begin{align}
\alpha> \sqrt{2\omega}+\beta\sqrt{1-\omega}+(1+\beta)\sqrt{\epsilon+\beta/3}, \label{eq:PZF-probabilistic-gap}
\end{align} then there exists a non-empty interval $\Lambda \subset \Re$ such that for any $\lambda \in \Lambda$, any optimal solution to the PZF-SSC problem \eqref{eq:Lasso-PZF} is non-zero and subspace preserving, with probability at least $1-2/N^2-\exp(-\sqrt{\rho} d) - \exp(-D\epsilon/2 )$. 
\end{thm}

\noindent To get an insight into how the maximal tolerable level of missing entries scales with the subspace dimension $d$, we note that for high-ambient dimensions $D$ the quantity $\beta$ is negligible with respect to the quantity $\alpha$. Similarly, ignoring the small parameter $\epsilon$, \eqref{eq:PZF-probabilistic-gap} becomes approximately $\alpha \ge \sqrt{2 \omega}$, which by the definition of $\alpha$ and $\omega$ gives
\begin{align}
\text{PZF-SSC}: \, \, \, \frac{m}{D} < \frac{1}{2}\frac{\log(\rho)}{16d}=\mathcal{O}\left(\frac{1}{d}\right). \label{eq:PZF-rate}
\end{align} Informally, \eqref{eq:PZF-rate} says that the maximal tolerable percentage of missing entries of PZF-SSC as predicted by Theorem \ref{thm:PZF-probabilistic}, scales inversely proportionally to the subspace dimension.

\subsection{ZF-SSC Theory} \label{subsection:ZF-SSC}

Similar techniques that led to Theorems \ref{thm:PZF-deterministic} and \ref{thm:PZF-probabilistic} can be employed to yield deterministic and probabilistic statements about ZF-SSC. In particular, we have:

\begin{thm}[ZF-SSC, deterministic] \label{thm:ZF-deterministic}
With the notation of Definition \ref{dfn:BasicObjects}, further define the positive quantity
\small
\begin{align}
\bar{\lambda}_{\max} &:= \frac{1}{2}\Bigg\{\frac{1}{2 \bar{\zeta}} -\frac{\bar{\mu}_{\lambda}}{\bar{\gamma}\bar{\eta}}-\frac{1}{2 \bar{\eta}^2}+ 
\Bigg(\frac{9}{4\bar{\zeta}^2}+\frac{\bar{\mu}_{\lambda}}{\bar{\gamma}\bar{\eta}\bar{\zeta}}+ \nonumber\\
& \frac{2}{\bar{\gamma}\bar{\eta}^2} + \frac{\bar{\mu}_{\lambda}^2}{\bar{\gamma}^2 \bar{\eta}^2}+ \frac{1}{4 \bar{\eta}^4}+
\frac{1}{\bar{\eta}^2}\Big(\frac{\bar{\mu}_{\lambda}}{\bar{\gamma} \bar{\eta}}-
\frac{1}{2 \bar{\zeta}}\Big) \Bigg)^{1/2} \Bigg\}. \label{eq:lambda-u}
\end{align} \normalsize If the condition
\begin{align}
\bar{\mu}_{\lambda} \,\bar{\eta} + \bar{\gamma} < \bar{\zeta} \label{eq:ZF-deterministic-gap}
\end{align} is true, then the open interval $\bar{\Lambda} := (1/\bar{\zeta}, \bar{\lambda}_{\max})$
is non-empty. If in addition $\lambda \in \bar{\Lambda}$, then every optimal solution to the Lasso SSC problem
\eqref{eq:Lasso-ZF} with zero-filled data is non-zero and subspace preserving.
\end{thm} 

\noindent The quantities $\bar{\mu}_{\lambda}, \, \bar{\eta}, \, \bar{\zeta}$ are in direct analogy with the quantities $\dot{\bar{\mu}}_{\lambda}, \, \dot{\bar{\eta}}, \, \dot{\bar{\zeta}}$ that appeared in Theorem \ref{thm:PZF-deterministic}, except that now they are defined in terms of ZF data instead of PZF data. In fact, as seen from their definitions in \eqref{eq:eta} and \eqref{eq:zeta}, $\bar{\eta} = \dot{\bar{\eta}}$ and $\bar{\zeta}=\dot{\bar{\zeta}}$, while in principle the inter-subspace coherences $\bar{\mu}_{\lambda}, \dot{\bar{\mu}}_{\lambda}$ need not coincide. Instead, the main difference between \eqref{eq:ZF-deterministic-gap} and \eqref{eq:PZF-deterministic-gap} is the appearance of the quantity $\bar{\gamma}$, whose PZF counterpart $\dot{\bar{\gamma}}$ appears in Theorem \ref{thm:PZF-deterministic} only in the definition of the allowable interval for $\lambda$. 

The quantity $\bar{\gamma}$ admits an interesting interpretation: As seen from its definition in \eqref{eq:gamma-ZF}, $\bar{\gamma}$ is capturing the coherence between the ZF data $\bar{\bX}^{(-1)}$ associated to subspaces $\cS_i, i>1$, and a projected version of the unobserved components $\tilde{\bX}_{-1}^{(1)}$ of the data from $\cS_1$. A large such coherence intuitively means that significant information about $\cS_1$, potentially crucial for the reconstruction of $\bar{\x}_1^{(1)}$ as a linear combination of points in $\bar{\bX}_{-1}$, is \emph{leaked away} into $\tilde{\bX}_{-1}^{(1)}$, with which $\bar{\bX}^{(-1)}$ highly correlates (assuming large $\bar{\gamma}$). In turn, this may lead the optimization problem to favor points of $\bar{\bX}^{(-1)}$ in expressing $\bar{\x}_1^{(1)}$, thus leading to the loss of the subspace-preserving property of the solutions to \eqref{eq:Lasso-ZF}. 

Interestingly, comparison of the proofs of Theorems \ref{thm:PZF-deterministic} and \ref{thm:ZF-deterministic} reveals that $\dot{\bar{\gamma}}$ did not appear in \eqref{eq:PZF-deterministic-gap} because $\bar{\x}_1^{(1)}$ is complete when the underlying subspace arrangement is taken to be $\bigcup_{i=1}^n \dot{\cS}_i$, which is the natural view that we adopted for our analysis of PZF-SSC. On the contrary, such a feature is not available in the analysis of ZF-SSC, as $\bar{\x}_1^{(1)}$ is in principle incomplete with respect to $\bigcup_{i=1}^n \cS_i$. 

\noindent As we did for PZF-SSC, we use the deterministic Theorem \ref{thm:ZF-deterministic} to derive a probabilistic statement:

\begin{thm}[ZF-SSC, probabilistic]\label{thm:ZF-probabilistic}
Consider the exact setting of Theorem \ref{thm:PZF-probabilistic}. If the ratio 
$\omega:=m/D$ of missing entries satisfies
\small
\begin{align}
\alpha>& (\sqrt{2}+\sqrt{\epsilon+\beta/3}) \sqrt{\omega}+(\beta+\sqrt{\epsilon+\beta/3})\sqrt{1-\omega}+ \nonumber\\
& \sqrt{\omega(1-\omega)}+(1+\beta+\sqrt{\epsilon+\beta/3})\sqrt{\epsilon+\beta/3}, \label{eq:ZF-probabilistic-gap}
\end{align} \normalsize then there exists a non-empty interval $\Lambda \subset \Re$ such that for any $\lambda \in \Lambda$, any optimal solution to the ZF-SSC problem \eqref{eq:Lasso-ZF} is non-zero and subspace preserving, with probability at least $1-2/N^2-\exp(-\sqrt{\rho} d) - \exp\Big(-D\epsilon/2 \Big)$. 
\end{thm}

\subsection{A Comparison between PZF-SSC and ZF-SSC} \label{subsection:Comparison}

\begin{figure}
\centering	
		
\includegraphics[width=1\linewidth]{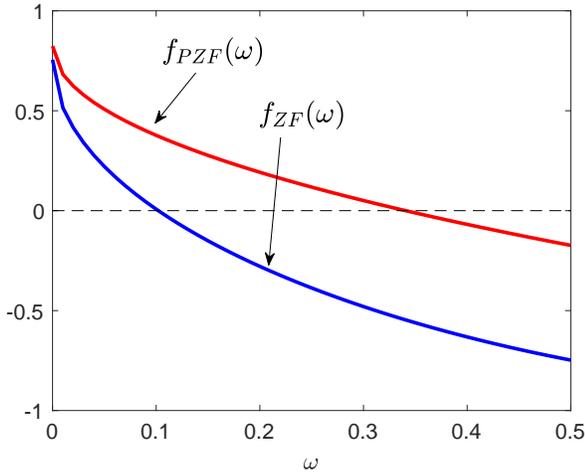}
\caption{{\bf Theoretical comparison between PZF-SSC and ZF-SSC:} The functions $f_{\text{PZF}}(\omega)$ and $f_{\text{ZF}}(\omega)$ defined in \eqref{eq:f-PZF} and \eqref{eq:f-ZF} respectively are plotted for $\alpha=0.9, \beta=0.01, \epsilon=0.001$.}
		
\label{fig:f}
\end{figure} 

It is worth mentioning that repeating the informal arguments that led to \eqref{eq:PZF-rate}, i.e., for high ambient dimension $D$ ignoring $\beta$ and $\epsilon$, \eqref{eq:ZF-probabilistic-gap} becomes 
$\alpha> \sqrt{2\omega}+\sqrt{\omega(1-\omega)}$. Noting that $\sqrt{\omega} \ge \sqrt{\omega(1-\omega)}$, we then have that this latter simplified condition is satisfied if the stronger condition $\alpha > (1+\sqrt{2}) \sqrt{\omega}$ is true. In turn, this implies that 
\begin{align}
\text{ZF-SSC}:\, \, \, \frac{m}{D} < \frac{1}{(1+\sqrt{2})^2}\frac{\log(\rho)}{16d}=\mathcal{O}\left(\frac{1}{d}\right). \label{eq:ZF-rate}
\end{align} In other words, both PZF-SSC and ZF-SSC give subspace preserving solutions as long as the ratio of missing entries scales as \footnote{This result is in agreement with the work of \cite{Charles:arXiv18}, who however only study ZF-SSC.} $1/d$. On the other hand, the multiplying constant associated to PZF-SSC is about $3$ times larger, a significant difference. Alternatively (and more rigorously), defining
\begin{align}
f_{\text{PZF}}(\omega) :=& \alpha - \sqrt{2\omega}-\beta\sqrt{1-\omega}\nonumber \\
&-(1+\beta)\sqrt{\epsilon+\beta/3}, \label{eq:f-PZF}
\end{align}  the PZF Theorem \ref{thm:PZF-probabilistic} asks that
\begin{align}
f_{\text{PZF}}(\omega)> 0, 
\end{align} while the ZF Theorem \ref{thm:ZF-probabilistic} asks that
\begin{align}
f_{\text{ZF}}(\omega):=& -\sqrt{\epsilon+\beta/3}(\sqrt{\omega}+\sqrt{1-\omega}+\sqrt{\epsilon+\beta/3})  \nonumber \\
& - \sqrt{\omega(1-\omega)} + f_{\text{PZF}}(\omega)  > 0. \label{eq:f-ZF}
\end{align} To appreciate the difference between the two functions $f_{\text{PZF}}(\omega)$ and 
$f_{\text{ZF}}(\omega)$, we plot them in Fig. \ref{fig:f}\footnote{We note here that the probabilistic lower bound $\alpha$ on the inradius $r$, even though of fundamental theoretical importance, is overly pessimistic for realistic values of $\rho, d$. Instead, we have used here the more realistic value $\alpha=0.9$. At any case, as seen by the definition of $f_{\text{PZF}}(\omega)$ and 
$f_{\text{ZF}}(\omega)$ in \eqref{eq:f-PZF} and \eqref{eq:f-ZF} respectively, the value of $\alpha$ does not in any way affect the behavior of the two functions, other than a uniform vertical translation. For an alternative bound see \cite{You:ICML15}.}, which suggests the superiority of PZF-SSC.

\section{New Results for Generic Noise and Uncorrupted Data} \label{section:OtherImplications}

Following the same methodology that led to Theorem \ref{thm:PZF-deterministic}, but now for data corrupted by generic, potentially adversarial, bounded noise, we obtain the following result:

\begin{thm}[SSC with generic noise, deterministic] \label{thm:SSCgenericNoise}
Given data $\bX$ as described in Definition \ref{dfn:BasicObjects}, suppose that generic deterministic noise $\boldsymbol{\delta}_k^{(i)}$ of maximum Euclidean norm $\delta$ is added at every point $\x_k^{(i)}$. For a given $\lambda$, let $\hat{\v}_{1,\lambda}'$ be the normalized projection onto $\cS_1$ of the optimal solution to the reduced dual problem associated to the noisy points $\x_j'^{(1)}:=\x_j^{(1)} + \boldsymbol{\delta}_j^{(1)}$ coming from subspace $\cS_1$ (i.e., use $\x_j'^{(1)}$ in \eqref{eq:dfn-LassDualReducedZF}, then compute the dual direction as for incomplete data). Let $\mu_{\lambda}'$ be the maximal absolute inner product between $\hat{\v}_{1,\lambda}'$ and points $\x_k^{(i)}, \, i>1$. If $r> \mu_{\lambda}'$ and if  
\begin{align}
\delta < r+\frac{\mu_{\lambda}'}{3} - \sqrt{\Big(r+\frac{\mu_{\lambda}'}{3}\Big)^2-\frac{r^2-(\mu_{\lambda}')^2}{3}}, \label{eq:NewNoisySSC-deterministic}
\end{align} then there exists a non-empty open interval $\Lambda'_{\lambda}$, such that if $\lambda \in \Lambda'_{\lambda}$, then every optimal solution to the Lasso SSC problem of expressing the noisy point $\x_1^{(1)} + \boldsymbol{\delta}_1^{(1)}$ in terms of all other noisy points, will be non-zero and subspace preserving. In particular, condition \eqref{eq:NewNoisySSC-deterministic} is satisfied if 
\begin{align}
\delta < \frac{r-\mu_{\lambda}'}{6}. \label{eq:NewNoisySSC-deterministic2}
\end{align}
\end{thm}

\noindent Theorem \ref{thm:SSCgenericNoise} is an improvement over Theorem $6$ in \cite{Wang:JMLR16}, which restricts $\delta$ as in
\begin{align}
\delta < \frac{r(r-\mu_{\lambda}')}{2+7r}.
\end{align} For Gaussian noise, it can be shown that Theorem \ref{thm:SSCgenericNoise} allows the variance of the noise to scale roughly as per $1/\sqrt{d}$, as opposed to $1/d$ for Theorem $11$ in \cite{Wang:JMLR16}. Even though this is an improvement over the latter, it is still not better than 
\cite{Soltanolkotabi:AS14}, which analyzes SSC with an adaptive tuning of $\lambda$, thus removing the dependence of the subspace dimension $d$ from the tolerable noise variance.

We conclude the paper with a result interesting in its own right, obtained from Theorem \ref{thm:PZF-deterministic} for zero missing entries:

\begin{thm}[SSC with uncorrupted data, deterministic] \label{thm:SSCuncorrupted}
Consider expressing point $\x_1^{(1)}$ in terms of the rest of the points in $\bX$ via the Lasso SSC formulation \eqref{eq:SSC-Lasso}. If  $\mu_{\lambda} < \zeta$ then the open interval 
\begin{align}
\Lambda_{\lambda} :=\Bigg(\frac{1}{\zeta},\frac{1}{2} \Bigg(\frac{1}{\zeta} + \frac{1}{\mu_{\lambda}} \Bigg) \Bigg) \label{eq:SSC-new-lambda-UB}
\end{align} is non-empty, and if $\lambda \in \Lambda_{\lambda}$, then any optimal solution is non-zero and subspace preserving.
\end{thm}

\section{Proofs} \label{section:Proofs}

In this section we give proofs of Theorem \ref{thm:SSCuncorruptedOld}\footnote{As noted earlier, this theorem follows from Theorem $6$ in \cite{Wang:JMLR16} for zero noise. Here we give a direct proof based on the framework developed in \cite{Soltanolkotabi:AS12} for basis pursuit and later expanded for Lasso by \cite{Wang:JMLR16}; the proof is instructive and highlights techniques relevant to other proofs in this paper.} as well as of Theorems \ref{thm:PZF-deterministic}-\ref{thm:SSCuncorrupted}. A key observation upon which the new Theorems \ref{thm:PZF-deterministic}-\ref{thm:SSCuncorrupted} rely, is the following lemma.

\begin{lem} \label{lem:New-vUB}
Let $\bY \in \Re^{D \times L}, \, \y \in \Re^D$, and define
\begin{align}
\zeta := \norm{\bY^\transpose \y}_{\infty}, \, \, \, \text{and} \, \, \, \eta := \norm{ \y }_2. 
\end{align} Let $\v^*$ be the unique solution to the strongly convex problem
\begin{align}
\max_{\v} \, \, \, \langle \y, \v \rangle - \frac{1}{2 \lambda} \norm{\v }_2^2 \, \, \, \text{s.t.} \, \, \, \norm{\bY^\top \v }_{\infty} \le 1. \label{eq:reduced}
\end{align} Suppose that $\lambda > 1/ \zeta$. Then 
\begin{align}
 \eta/\zeta \le \norm{\v^* }_2 \le \eta \left(2 \lambda- 1/\zeta\right). \label{eq:v-bound}
 \end{align} 
\end{lem}
\begin{proof}
Problem \eqref{eq:reduced} is equivalent to
\begin{align}
\min_{\v} \, \, \,\norm{\lambda \y -\v }_2^2 \, \, \, \text{s.t.} \, \, \, \norm{\bY^\transpose \v }_{\infty} \le 1. 
\label{eq:PZF-dual-distance}
\end{align} Decomposing $\v = \alpha \y + \bxi$, with $\bxi \perp \y$, \eqref{eq:PZF-dual-distance} becomes
\small
\begin{align}
\min_{\v} \, \, \,\norm{\lambda \y -\alpha \y }_2^2 +\norm{ \bxi }_2^2 \, \, \, \text{s.t.} \, \, \,  \norm{ \bY^\transpose(\alpha \y+\bxi )}_{\infty} \le 1.
\end{align} \normalsize Now, $\y/\zeta$ is a feasible point of \eqref{eq:PZF-dual-distance}, hence, decomposing the optimal solution $\v^*$ as $\v^*=\alpha^* \y + \bxi^*$, we have that
\begin{align}
\norm{\lambda \y -\alpha^* \y }_2^2 &\le \norm{\lambda \y -\alpha^* \y }_2^2 + \norm{ \bxi^* }_2^2 \\
&\le \norm{\lambda \y - \y/\zeta }_2^2. \end{align}
Remembering that $\lambda > 1/\zeta$, the above inequality implies 
\begin{align}
\left| \lambda - \alpha^* \right| \le \lambda - 1/\zeta \, \, \, \Leftrightarrow \, \, \,  1/\zeta \le \alpha^* \le 2 \lambda - 1/\zeta. \label{eq:Bound-alpha}
\end{align} Now this gives the lower bound in \eqref{eq:v-bound}, since 
\begin{align}
\norm{\v^* }_2^2  &= \norm{\alpha^* \y }_2^2 +  \norm{\bxi^* }_2^2  \\
& \ge \norm{\alpha^* \y }_2^2  = (\alpha^*)^2 \eta ^2 \stackrel{\eqref{eq:Bound-alpha}}{\ge} \left(\eta/ \zeta \right)^2.
\end{align} To get the upper bound in \eqref{eq:v-bound}, we work again with \begin{align}
& \norm{\lambda \y -\alpha^* \y }_2^2 + \norm{ \bxi^* }_2^2 \le \norm{\lambda \y - \y/\zeta}_2^2, \, \, \, \Leftrightarrow \\
& \eta^2 \lambda^2 -2 \lambda \eta^2 \alpha^* + \left\| \v^* \right\|^2 \le \eta^2 \left(\lambda  - 1/\zeta\right)^2, \Leftrightarrow \\
& \left\| \v^* \right\|^2 \le \eta^2 \left(\lambda  - 1/\zeta\right)^2 -\eta^2 \lambda^2 +2 \lambda \eta^2 \alpha^*. 
\end{align} Using the upper bound \eqref{eq:Bound-alpha} on $\alpha^*$, this becomes
\small
\begin{align}
\left\| \v^* \right\|^2 &\le \eta^2 \left(\lambda  - 1/\zeta\right)^2 -\eta^2 \lambda^2 +2 \lambda \eta^2 \left(2\lambda  - 1/\zeta\right) \\
& = \eta^2 \left(2\lambda  - 1/\zeta\right)^2,
\end{align} \normalsize by which we are done.
\end{proof}

\subsection{Proof of Theorem \ref{thm:SSCuncorruptedOld}} \label{subsection:ProofOld}

According to the general Lasso Lemma $12$ 
in \cite{Wang:JMLR16} (which is a generalization of the general basis pursuit Lemma $7.1$ in \cite{Soltanolkotabi:AS12}) the solution to \eqref{eq:SSC-Lasso} will be subspace preserving if the following subspace separation condition is true,
\begin{align}
\abs{ \langle \x_k^{(i)}, \v^*_{\lambda} \rangle } < 1, \, \, \, \forall i>1, \, \, \, \forall k \in [N_i]. \label{eq:ProofLassoSSC-SS}
\end{align} Now, it can be verified that $\v^*_{\lambda} \in \cS_1$, and so $\v^*_{\lambda} = \bP_{\cS_1}\v^*_{\lambda}=:\v_{1,\lambda}^*$ (otherwise $\bP_{\cS_1}\v^*_{\lambda}$ is a feasible point of the reduced dual problem with larger objective; contradiction). Hence, condition \eqref{eq:ProofLassoSSC-SS} is equivalent to 
\begin{align}
\norm{\v_{1,\lambda}^*}_2\abs{ \langle \x_k^{(i)}, \hat{\v}_{1,\lambda}^* \rangle } < 1, \, \, \, \forall i>1, \, \, \, \forall k \in [N_i]. \label{eq:ProofLassoSSC-SS-2}
\end{align} Now, $\v_{1,\lambda}^*$ is an element of the symmetrized relative polar set associated to $\bX_{-1}^{(1)}$, and so $\norm{\v_{1,\lambda}^*}_2 \le 1/r$ (see Definition \ref{dfn:InradiusCircumradiusPolar} and Lemma \ref{lem:r-R}). Hence condition 
\eqref{eq:ProofLassoSSC-SS-2} is satisfied, if the stronger condition $ \mu_{\lambda} < r$ is true. This guarantees that every solution to \eqref{eq:SSC-Lasso} is subspace preserving. Then the condition $\lambda > 1/ \zeta$ guarantees that any such solution is also non-zero, as under this hypothesis, examination of the KKT conditions reveals that $\c=\0$ is never an optimal solution.

\subsection{Proof of Theorem \ref{thm:PZF-deterministic}} \label{subsection:ProofPZFdeterministic}

Consider the reduced Lasso SSC problem
\begin{align}
\min_{\c} \, \, \, \norm{\c}_1 + \frac{\lambda}{2} \norm{\e}_2^2, \, \, \, \text{s.t.} \, \, \, \e = \bar{\x}_1^{(1)}-  \dot{\bar{\bX}}_{-1}^{(1)}\c, \label{eq:PZF-SSC-reduced}
\end{align} as well as its dual optimization problem
\small
\begin{align}
\max_{\v} \, \, \, \langle \bar{\x}_1^{(1)}, \v \rangle - \frac{1}{2 \lambda} \left\|\v \right\|_2^2 \, \, \, \text{s.t.} \, \, \, \norm{(\dot{\bar{\bX}}_{-1}^{(1)}) ^{\transpose} \v}_{\infty} \le 1. \label{eq:Lasso-dual-noisy}
\end{align} \normalsize Problem \eqref{eq:Lasso-dual-noisy} is strongly convex; let
$\dot{\bar{\v}}^*$ be its unique solution. Then by Lemma $12$ in \cite{Wang:JMLR16}, every solution to the Lasso SSC problem \eqref{eq:Lasso-PZF} will be subspace preserving if
\begin{align}
\abs{ \langle \dot{\bar{\x}}_k^{(i)}, \dot{\bar{\v}}^* \rangle } < 1, \, \, \, \forall i>1, \, \, \, \forall k \in [N_i], \label{eq:Proof-PZF-geometric-SS}
\end{align} Condition \eqref{eq:Proof-PZF-geometric-SS} is satisfied, if the following stronger but more structured condition is satisfied:
\small
\begin{align}
\abs{ \langle \dot{\bar{\x}}_k^{(i)}, \bP_{\dot{\cS}_1}\dot{\bar{\v}}^*) \rangle } + \abs{ \langle \dot{\bar{\x}}_k^{(i)}, \bP_{\dot{\cS}_1^\perp} \dot{\bar{\v}}^* \rangle } < 1, \, \, \, \forall i>1, \, \forall k. \label{eq:Proof-PZF-geometric-SS-structured}
\end{align} \normalsize By thinking of $\v$ as a Lagrange multiplier associated with the equality constraint $\bar{\x}_1^{(1)} = \dot{\bar{\bX}}_{-1}^{(1)} \c + \e$ of \eqref{eq:PZF-SSC-reduced}, we see from the KKT optimality conditions that 
\begin{align}
\dot{\bar{\v}}^* = \lambda \dot{\bar{\e}}^* = \lambda \left(\bar{{\x}}_1^{(1)} - \dot{\bar{\bX}}_{-1}^{(1)} \dot{\bar{\c}}^* \right), \label{eq:Proof-PZF-SSC-deterministic-v-star}
\end{align} where $\dot{\bar{\e}}^*, \dot{\bar{\c}}^*$ are optimal solutions of \eqref{eq:PZF-SSC-reduced}. Taking into consideration that 
\begin{align}
\bar{\x}_1^{(1)} \in \dot{\cS}_1 \, \, \, \text{and} \, \, \,  \dot{\bar{\bX}}_{-1}^{(1)} = \dot{\bX}_{-1}^{(1)} - \dot{\tilde{\bX}}_{-1}^{(1)},
\end{align} projecting \eqref{eq:Proof-PZF-SSC-deterministic-v-star} onto $\dot{\cS}_1^\perp$ gives 
\begin{align}
\bP_{\dot{\cS}_1^\perp} \dot{\bar{\v}}^* &= - \lambda \bP_{\dot{\cS}_1^\perp} \dot{\bar{\bX}}_{-1}^{(1)} \dot{\bar{\c}}^*\\
&= \lambda \bP_{\dot{\cS}_1^\perp} \dot{\tilde{\bX}}_{-1}^{(1)} \dot{\bar{\c}}^* \\
&= \lambda \sum_{j=2}^{N_1} \dot{\bar{\c}}^*_j \, \bP_{\dot{\cS}_1^\perp} \dot{\tilde{\x}}_{j}^{(1)}. 
\end{align} Hence,
\begin{align}
&\abs{\langle \dot{\bar{\x}}_k^{(i)}, \bP_{\dot{\cS}_1}^\perp \dot{\bar{\v}}^* \rangle} = 
\lambda
\abs[\Big]{ \sum_{j=2}^{N_1} \dot{\bar{\c}}_j^* \langle \dot{\bar{\x}}_k^{(i)},  \bP_{\dot{\cS}_1}^\perp \dot{\tilde{\x}}_j^{(1)} \rangle}\\
& \le \lambda \Bigg( \max_{\substack{i>1,k \in [N_i],\\ j \in [N_1]}} \abs{\langle \dot{\bar{\x}}_k^{(i)},  \bP_{\dot{\cS}_1}^\perp \dot{\tilde{\x}}_j^{(1)} \rangle} \Bigg) \norm{\dot{\bar{\c}}^*}_1 \\
&= \lambda \dot{\bar{\gamma}} \norm{\dot{\bar{\c}}^*}_1.
\label{eq:PZF-v-UB-c} 
\end{align}  Let us derive an upper bound on $\norm{\dot{\bar{\c}}^*}_1$. 
By strong duality 
\begin{align}
\norm{ \dot{\bar{\c}}}_1^* &+ \frac{\lambda}{2} \norm{\dot{\bar{\e}}^*}_2^2 = \langle \bar{\x}_1^{(1)}, \dot{\bar{\v}}^* \rangle - \frac{1}{2 \lambda} \norm{\dot{\bar{\v}}^*}_2^2 \, \Rightarrow \\
\norm{ \dot{\bar{\c}}^* }_1 &= \langle  \bar{\x}_1^{(1)}, \dot{\bar{\v}}^* \rangle - \frac{1}{2\lambda} \norm{ \dot{\bar{\v}}^*}_2^2 - \frac{\lambda}{2} \norm{\dot{\bar{\e}}^*}_2^2 \\
&  \stackrel{\dot{\bar{\v}}^* = \lambda \dot{\bar{\e}}^*}{=} \langle  \bar{\x}_1^{(1)}, \dot{\bar{\v}}^* \rangle - \frac{1}{\lambda} \norm{ \dot{\bar{\v}}^*}_2^2 \\
& =   \langle  \bP_{\dot{\cS}_1} \bar{\x}_1^{(1)},  \dot{\bar{\v}}^* \rangle - \frac{1}{\lambda} \norm{ \dot{\bar{\v}}^*}_2^2 \\
& =   \langle  \bar{\x}_1^{(1)},  \bP_{\dot{\cS}_1} \dot{\bar{\v}}^* \rangle - \frac{1}{\lambda} \norm{ \dot{\bar{\v}}^*}_2^2 \\
& =   \langle  \bar{\x}_1^{(1)},  \dot{\bar{\v}}^*_1 \rangle - \frac{1}{\lambda} \norm{ \dot{\bar{\v}}^*}_2^2 \\
& \le \norm{ \bar{\x}_1^{(1)} }_2 \norm{ \dot{\bar{\v}}^*_1}_2 - \frac{1}{\lambda} \norm{ \dot{\bar{\v}}^*}_2^2 \\
&\stackrel{Lem. \ref{lem:New-vUB}}{\le} \dot{\bar{\eta}} \norm{  \dot{\bar{\v}}^*_1 }_2 - \dot{\bar{\eta}}^2\lambda^{-1} \dot{\bar{\zeta}}^{-2}. \label{eq:PZF-c-UB}
\end{align} Substituting \eqref{eq:PZF-c-UB} into \eqref{eq:PZF-v-UB-c}, and the result into \eqref{eq:Proof-PZF-geometric-SS-structured}, we obtain the sufficient condition
\begin{align}
 \, \dot{\bar{\mu}}_{\lambda} \, \norm{  \dot{\bar{\v}}^*_1 }_2
+ \lambda \, \dot{\bar{\gamma}} \, \big(\dot{\bar{\eta}}\, \norm{  \dot{\bar{\v}}^*}_1-\dot{\bar{\eta}}^2 / (\lambda \, \dot{\bar{\zeta}}^2)\big) < 1. \label{eq:Proof-PZF-deterministic-semifinal-SS}
\end{align}  By Lemma \ref{lem:New-vUB} we have
\begin{align}
\norm{  \dot{\bar{\v}}^*_1 }_2 \le \norm{  \dot{\bar{\v}}^* }_2 \le \dot{\bar{\eta}} (2 \lambda - 1/\dot{\bar{\zeta}}).
\end{align} Substituting this upper bound for $\norm{  \dot{\bar{\v}}^*_1 }_2$ into \eqref{eq:Proof-PZF-deterministic-semifinal-SS} (and assuming that $\dot{\bar{\gamma}}>0$) we obtain
\small
\begin{align}
\lambda^2 +\Big(\frac{ \dot{\bar{\mu}}_{\lambda}}{\dot{\bar{\eta}}\dot{\bar{\gamma}}} - \frac{1}{2\dot{\bar{\zeta}}} \Big) \lambda -\Big(\frac{1}{2 \dot{\bar{\eta}}^2 \dot{\bar{\gamma}}} + \frac{1}{2\dot{\bar{\zeta}}^2} + \frac{ \dot{\bar{\mu}}_{\lambda}}{2\dot{\bar{\eta}}\dot{\bar{\gamma}}\bar{\zeta}} \Big) <0. \label{eq:PZF-quadratic-lambda}
\end{align} \normalsize For the moment let us treat $\dot{\bar{\mu}}_{\lambda}$ as constant. Then regardless of the sign of the coefficient $\frac{ \dot{\bar{\mu}}_{\lambda}}{\dot{\bar{\eta}}\dot{\bar{\gamma}}} - \frac{1}{2\dot{\bar{\zeta}}}$, by Descartes's rule of signs the quadratic polynomial appearing in the LHS of \eqref{eq:PZF-quadratic-lambda} has exactly one negative and one positive root; the positive root is given precisely by $\dot{\bar{\lambda}}$ in \eqref{eq:lambda-u}. Thus constraining $\lambda$ to lie in the interval $(0,\dot{\bar{\lambda}})$ guarantees that \eqref{eq:PZF-quadratic-lambda} is true, which in turn guarantees that \eqref{eq:Proof-PZF-geometric-SS} is true, which finally guarantees that $\dot{\bar{\c}}^*$ is a subspace preserving solution. We also need to ensure that $\dot{\bar{\c}}^*$ is not the zero solution. Inspection of the KKT conditions for the reduced problem \eqref{eq:PZF-SSC-reduced} reveals that if $\lambda> 1/\dot{\bar{\zeta}}$, then any optimal solution of \eqref{eq:PZF-SSC-reduced} must necessarily be non-zero. In turn, this implies that any optimal solution of \eqref{eq:Lasso-PZF} must be non-zero. As a consequence, if $\lambda$ satisfies \begin{align}
1/\dot{\bar{\zeta}} < \lambda < \dot{\bar{\lambda}}, \label{eq:lambda-inequalities}
\end{align} then any optimal solution of \eqref{eq:Lasso-PZF} will be both non-zero and subspace preserving. For \eqref{eq:lambda-inequalities} to be true we must have that  
$1/\dot{\bar{\zeta}} < \dot{\bar{\lambda}}$. Substituting the expression for $\dot{\bar{\lambda}}$ into this last inequality, a straightforward algebraic manipulation reveals that $1/\dot{\bar{\zeta}} < \dot{\bar{\lambda}}$ is true as long as \eqref{eq:PZF-deterministic-gap} holds true.

\subsection{Proof of Theorem \ref{thm:PZF-probabilistic}} \label{subsection:ProofPZFprobabilistic}

First note that, under the fully random model, $\x_j^{(i)}$ is uniformly distributed\footnote{See \cite{Soltanolkotabi:AS12} for a detailed explanation.} on the unit sphere of $\Re^D$, for every $i \in [n], \, j \in [N_i]$. Define $\omega := m/D$. We bound the terms appearing in \eqref{eq:PZF-deterministic-gap} with high probability. In particular, since $\bar{\x}_1^{(1)}$ is the orthogonal projection of $\x_1^{(1)}$ onto the $(D-m)$-dimensional linear subspace $\E_1^{(1)}$, Lemma \ref{lem:RandomProjection} gives that
\begin{align}
&\mathbb{P}\Big[\dot{\bar{\eta}} \le
 \sqrt{1-\omega} + \sqrt{\epsilon+\beta/3} \Big] \ge \nonumber \\ 
 & 1 - 
\frac{1}{N}\exp\Big(-\frac{D\epsilon}{2} \Big). 
\end{align}  Moreover, for every $k,i$ we can write 
\begin{align}
\langle \dot{\bar{\x}}_k^{(i)}, \hat{\dot{\bar{\v}}}_1^* \rangle  &= 
\langle \bar{\bP}_1^{(1)} \bar{\bP}_k^{(i)} \x_k^{(i)}, \hat{\dot{\bar{\v}}}_1^* \rangle \\
&= \langle \x_k^{(i)}, \bar{\bP}_1^{(1)} \bar{\bP}_k^{(i)} \hat{\dot{\bar{\v}}}_1^* \rangle \\
& \le \langle \x_k^{(i)}, \widehat{\bar{\bP}_1^{(1)} \bar{\bP}_k^{(i)} \hat{\dot{\bar{\v}}}_1^*} \rangle.
\end{align} Since for $i>1$ we have that $\x_k^{(i)}$ is independent from $\widehat{\bar{\bP}_1^{(1)} \bar{\bP}_k^{(i)} \hat{\dot{\bar{\v}}}_1^*}$, Lemma \ref{lem:inner} together with the union bound give
\begin{align}
\mathbb{P}\Big[\dot{\bar{\mu}}_{\lambda} \le \beta \Big] \ge 1 - 
2/N^2. 
\end{align} Next, we bound from below $\dot{\bar{\zeta}}$ as
\begin{align}
\dot{\bar{\zeta}} \ge r - \max_{j>1}\norm{(\boldsymbol{I}-\bar{\bP}_{j}^{(1)}  \bar{\bP}_1^{(1)}) \x_{1}^{(1)}}_2. \label{eq:Proof-PZF-Probabilistic-zetaLB}
\end{align} To prove this, we use the technique of \cite{Wang:JMLR16}, while exploiting the fact that the noise is now given by projections. Specifically, let $j^*, j^\dagger$ be such that 
\begin{align}
\norm{(\dot{\bar{\bX}}_{-1}^{(1)})^\top \bar{\x}_{1}^{(1)}}_{\infty} &= 
\abs{\langle \dot{\bar{\x}}_{j^*}^{(1)}, \bar{\x}_{1}^{(1)} \rangle}, \\
\norm{(\bX_{-1}^{(1)})^\top \x_{1}^{(1)}}_{\infty} &= 
\abs{\langle \x_{j^\dagger}^{(1)}, \x_{1}^{(1)} \rangle}.
\end{align} Then 
\small
\begin{align}
\dot{\bar{\zeta}}&=\norm{(\dot{\bar{\bX}}_{-1}^{(1)})^\top \bar{\x}_{1}^{(1)}}_{\infty} \\
&= \abs{\langle \dot{\bar{\x}}_{j^*}^{(1)}, \bar{\x}_{1}^{(1)} \rangle} \\
&\ge  \abs{\langle \dot{\bar{\x}}_{j^\dagger}^{(1)}, \bar{\x}_{1}^{(1)} \rangle} \\
&= \abs{\langle \bar{\bP}_1^{(1)} \bar{\bP}_{j^\dagger}^{(1)} \x_{j^\dagger}^{(1)}, \bar{\bP}_1^{(1)} \x_{1}^{(1)} \rangle} \\
&= \abs{\langle \x_{j^\dagger}^{(1)},  \bar{\bP}_{j^\dagger}^{(1)}  \bar{\bP}_1^{(1)} \x_{1}^{(1)} \rangle} \\
&= \abs{\langle \x_{j^\dagger}^{(1)},   \x_{1}^{(1)} - (\boldsymbol{I}-\bar{\bP}_{j^\dagger}^{(1)}  \bar{\bP}_1^{(1)}) \x_{1}^{(1)}) \rangle} \\
& \ge \abs{\langle \x_{j^\dagger}^{(1)},   \x_{1}^{(1)} \rangle} - 
\abs{\langle \x_{j^\dagger}^{(1)},   (\boldsymbol{I}-\bar{\bP}_{j^\dagger}^{(1)}  \bar{\bP}_1^{(1)}) \x_{1}^{(1)} \rangle} \\
& \ge \abs{\langle \x_{j^\dagger}^{(1)},   \x_{1}^{(1)} \rangle} - 
\norm{(\boldsymbol{I}-\bar{\bP}_{j^\dagger}^{(1)}  \bar{\bP}_1^{(1)}) \x_{1}^{(1)}}_2 \\
& \stackrel{\text{Prp.} \ref{prp:r-charecterization}}{\ge} r - 
\norm{(\boldsymbol{I}-\bar{\bP}_{j^\dagger}^{(1)}  \bar{\bP}_1^{(1)}) \x_{1}^{(1)}}_2 \\
& \ge r - 
\max_{j>1} \norm{(\boldsymbol{I}-\bar{\bP}_j^{(1)}  \bar{\bP}_1^{(1)}) \x_{1}^{(1)}}_2.
\end{align} \normalsize Now, by Lemma \ref{lem:r}, 
\begin{align}
\mathbb{P}[r \ge\alpha] \ge 1- \exp{(-\sqrt{\rho} d)}. 
\end{align} Moreover, for every $j$ the vector $(\boldsymbol{I}-\bar{\bP}_{j}^{(1)}  \bar{\bP}_1^{(1)}) \x_{1}^{(1)}$ is the orthogonal projection of $\x_{1}^{(1)}$ onto a linear subspace of dimension at most $2m$. Hence, Lemma \ref{lem:RandomProjection} and the union bound give that 
\begin{align}
&\mathbb{P}\Big[\max_{j>1} \norm{(\boldsymbol{I}-\bar{\bP}_{j}^{(1)}  \bar{\bP}_1^{(1)}) \x_{1}^{(1)}}_2 \le \sqrt{2\omega} + \sqrt{\epsilon+\beta/3} \Big] \ge \nonumber \\
&\ge 1 - \frac{N_1-1}{N}\exp\Big(-\frac{D \epsilon}{2} \Big). 
\end{align} Putting everything together via the union bound, we have that
\eqref{eq:PZF-deterministic-gap} holds true with probability at least
\begin{align}
1-2/N^2-\exp(-\sqrt{\rho} d) - \exp\Big(-\frac{D \epsilon}{2} \Big),
\end{align}  as long as \eqref{eq:PZF-probabilistic-gap} is true.

\subsection{Proof of Theorem \ref{thm:ZF-deterministic}} \label{subsection:ProofZFdeterministic}

The proof follows exactly the same steps as the proof of Theorem \ref{thm:PZF-deterministic} with one crucial difference: one now works with the
original subspaces $\cS_i$, instead of the projected subspaces $\dot{\cS}_i$. In that case, the point $\bar{\x}_1^{(1)}$ being expressed has in principle non-zero component in $\cS_1^\perp$; contrast this to the fact that $\bar{\x}_1^{(1)}$ has zero component in $\dot{\cS}_i^\perp$. We give some details.
Consider the reduced Lasso SSC problem
\begin{align}
\min_{\c} \, \, \, \norm{\c}_1 + \frac{\lambda}{2} \norm{\e}_2^2, \, \, \, \text{s.t.} \, \, \, \e = \bar{\x}_1^{(1)}-  \bar{\bX}_{-1}^{(1)}\c, \label{eq:ZF-SSC-reduced}
\end{align} as well as its dual optimization problem
\small
\begin{align}
\max_{\v} \, \, \, \langle \bar{\x}_1^{(1)}, \v \rangle - \frac{1}{2 \lambda} \left\|\v \right\|_2^2 \, \, \, \text{s.t.} \, \, \, \norm{(\bar{\bX}_{-1}^{(1)}) ^{\transpose} \v}_{\infty} \le 1. \label{eq:Proof-ZF-deterministic-Lasso-dual-reduced}
\end{align} \normalsize Letting $\bar{\v}^*$ be the unique solution of \eqref{eq:Proof-ZF-deterministic-Lasso-dual-reduced}, by Lemma $12$ in \cite{Wang:JMLR16} we have that every solution to the Lasso SSC problem \eqref{eq:Lasso-ZF} will be subspace preserving if
\begin{align}
\abs{ \langle \bar{\x}_k^{(i)},\bar{\v}^* \rangle } < 1, \, \, \, \forall i>1, \, \, \, \forall k \in [N_i], \label{eq:Proof-ZF-geometric-SS}
\end{align} or if the stronger condition
\small
\begin{align}
\abs{ \langle \bar{\x}_k^{(i)}, \bP_{\cS_1}\bar{\v}^*) \rangle } + \abs{ \langle \bar{\x}_k^{(i)}, \bP_{\cS_1^\perp} \bar{\v}^* \rangle } < 1,  \forall i>1,  \forall k,\label{eq:Proof-ZF-geometric-SS-structured}
\end{align} \normalsize is satisfied. Letting $\bar{\e}^*, \bar{\c}^*$ be optimal solutions of \eqref{eq:ZF-SSC-reduced}, we have that
\begin{align}
\bar{\v}^* = \lambda \bar{\e}^* = \lambda \left(\bar{{\x}}_1^{(1)} - \bar{\bX}_{-1}^{(1)} \bar{\c}^* \right), \label{eq:Proof-ZF-SSC-deterministic-v-star}
\end{align} and so
\begin{align}
&\abs{\langle \bar{\x}_k^{(i)}, \bP_{\cS_1^\perp} \bar{\v}^* \rangle} = 
\lambda \Big| \langle \bar{\x}_k^{(i)},  \bP_{\cS_1^\perp} \tilde{\x}_1^{(1)} \rangle + \nonumber \\
& \sum_{j=2}^{N_1} \bar{\c}_j^* \langle \bar{\x}_k^{(i)},  \bP_{\cS_1^\perp} \tilde{\x}_j^{(1)} \rangle \Big| \\
&\le \lambda \bar{\gamma} (1+ \norm{\bar{\c}^*}_1).
\label{eq:ZF-v-UB-c} 
\end{align} In a similar way as in the proof of Theorem \ref{thm:PZF-deterministic}, we have
\begin{align}
\norm{ \bar{\c}^* }_1 \le \bar{\eta} \norm{ \bar{\v}^*_1 }_2 - \bar{\eta}^2\lambda^{-1} \bar{\zeta}^{-2}. \label{eq:ZF-c-UB}
\end{align} Substituting \eqref{eq:ZF-c-UB} into \eqref{eq:ZF-v-UB-c}, and the result into \eqref{eq:Proof-ZF-geometric-SS-structured}, and using Lemma \ref{lem:New-vUB}, we obtain the sufficient condition \footnote{Observe that the only structural difference between \eqref{eq:ZF-quadratic-lambda} and \eqref{eq:PZF-quadratic-lambda} is the extra term $\frac{1}{2\bar{\eta}^2}$ appearing in \eqref{eq:ZF-quadratic-lambda}.}
\small
\begin{align}
\lambda^2 +\Big(\frac{\bar{\mu}_{\lambda}}{\bar{\eta}\bar{\gamma}} +\frac{1}{2\bar{\eta}^2}- \frac{1}{2\bar{\zeta}} \Big) \lambda -\Big(\frac{1}{2 \bar{\eta}^2 \bar{\gamma}} + \frac{1}{2\bar{\zeta}^2} + \frac{\bar{\mu}_{\lambda}}{2\dot{\eta}\bar{\gamma}\bar{\zeta}} \Big) <0. \label{eq:ZF-quadratic-lambda}
\end{align} \normalsize The LHS is a quadratic polynomial in $\lambda$ (treating $\bar{\mu}_{\lambda}$ as a constant) with one negative and one positive root, the latter being precisely $\bar{\lambda}$. If  
\begin{align}
1/\bar{\zeta} < \lambda < \bar{\lambda}, \label{eq:ZF-lambda-inequalities}
\end{align} then any optimal solution of \eqref{eq:Lasso-ZF} will be both non-zero and subspace preserving. The condition \eqref{eq:ZF-deterministic-gap} ensures that 
\begin{align}
1/\bar{\zeta} < \bar{\lambda}. \label{eq:lambda-existence}
\end{align}

\subsection{Proof of Theorem \ref{thm:ZF-probabilistic}} \label{subsection:ProofZFprobabilistic}

We bound with high probability the terms in \eqref{eq:ZF-deterministic-gap}. The terms $\bar{\zeta}, \bar{\mu}_{\lambda}, \bar{\eta}$ are bounded exactly as in the proof of Theorem \ref{thm:PZF-probabilistic}. For the term $\bar{\gamma}$, note that 
\begin{align}
\abs{\langle \bar{\x}_k^{(i)}, \bP_{\cS_1^\perp} \tilde{\x}_j^{(1)} \rangle} &\le \norm{\bar{\x}_k^{(i)}}_2 \norm{\bP_{\cS_1^\perp} \tilde{\x}_j^{(1)}}_2 \\
& \le \norm{\bar{\x}_k^{(i)}}_2 \norm{\tilde{\x}_j^{(1)}}_2.
\end{align} $\bar{\x}_k^{(i)}$ is the orthogonal projection of $\x_k^{(i)}$ onto the linear subspace $\E_k^{(i)}$ of dimension $D-m$, while $\tilde{\x}_j^{(1)}$ is the orthogonal projection of $\x_j^{(1)}$ onto the linear subspace $\tilde{\E}_j^{(1)}$ of dimension $m$. Hence, Lemma \ref{lem:RandomProjection} and the union bound give
\begin{align}
&\mathbb{P}\Big[ \bar{\gamma} \le (\sqrt{1-\omega}+\sqrt{\epsilon+\beta/3})(\sqrt{\omega}+\sqrt{\epsilon+\beta/3})\Big] \ge \nonumber \\
& 1- \exp\Big(-\frac{D\epsilon}{2} \Big).
\end{align} Replacing each term with its corresponding upper or lower bound in 
\eqref{eq:ZF-deterministic-gap} gives precisely \eqref{eq:ZF-probabilistic-gap}.

\subsection{Proof of Theorem \ref{thm:SSCgenericNoise}} \label{subsection:ProofGeneric}

We only note that the entire derivation of Theorem \ref{thm:ZF-deterministic} can be repeated to arrive at identical formulas, except that now all quantities are to be computed using 
$\bar{\x}_k^{(i)} = \x_k^{(i)} + \boldsymbol{\delta}_k^{(i)}$ and $\tilde{\x}_k^{(i)} =  -\boldsymbol{\delta}_k^{(i)}$. Specifically, we now have the following bounds:
\begin{align}
\bar{\zeta} &\ge r -2 \delta -\delta^2, \\
\bar{\mu}_{\lambda} & \le \mu_{\lambda}' + \delta, \\
\bar{\eta} & \le \sqrt{1+\delta^2}, \\
\bar{\gamma} & \le \delta + \delta^2.
\end{align} Hence, condition \eqref{eq:ZF-deterministic-gap} is satisfied if
\begin{align}
r -2 \delta -\delta^2 > \sqrt{1+\delta^2} (\mu_{\lambda}' + \delta) + \delta + \delta^2. \label{eq:ProofGenericNoiseSSdelta}
\end{align} Assuming that 
\begin{align}
r- 3 \delta -2\delta^2>0 \Leftrightarrow \delta < \frac{\sqrt{9+8r}-3}{4}, \label{eq:ProofGenericNoiseAssumption}
\end{align} \eqref{eq:ProofGenericNoiseSSdelta} is equivalent to
\begin{align}
(r> 3 \delta -2\delta^2)^2 > (1+\delta^2) (\mu_{\lambda}' + \delta)^2,
\end{align} which in turn is equivalent to 
\begin{align}
&-3\delta^4+2(\mu_{\lambda}'-6)\delta^3+((\mu_{\lambda}')^2-8+4r)\delta^2+ \nonumber \\
&2(\mu_{\lambda}'+3r)\delta + (\mu_{\lambda}')^2-r^2<0.
\end{align} This latter condition is true if the following stronger condition is true:
\begin{align}
-3\delta^2+ 2(\mu_{\lambda}'+3r)\delta + (\mu_{\lambda}')^2-r^2<0. \label{eq:ProofGenericNoisedelta}
\end{align} The LHS is a quadratic polynomial in $\delta$ with two positive roots. Condition \eqref{eq:NewNoisySSC-deterministic} forces $\delta$ to be smaller than the smallest positive root. Finally, \eqref{eq:ProofGenericNoisedelta} can be further relaxed to the even stronger (but linear) condition
\begin{align}
2(\mu_{\lambda}'+3r)\delta + (\mu_{\lambda}')^2-r^2<0,
\end{align} which is implied by\footnote{Note also that \eqref{eq:ProofGenericNoiseAssumption} is implied by \eqref{eq:NewNoisySSC-deterministic}.} \eqref{eq:NewNoisySSC-deterministic2}.

\subsection{Proof of Theorem \ref{thm:SSCuncorrupted}} \label{subsection:ProofSSCuncorrupted}

Replace $\norm{\v_{1,\lambda}^*}_2$ in \eqref{eq:ProofLassoSSC-SS-2} by the upper bound \eqref{eq:v-bound} of Lemma \ref{lem:New-vUB} to obtain 
\begin{align}
(2\lambda -1/\zeta)\abs{ \langle \x_k^{(i)}, \hat{\v}_{1,\lambda}^* \rangle } &< 1, \, \, \, \forall i>1, \forall k \, \, \, \Leftrightarrow\\
(2\lambda -1/\zeta) \mu_{\lambda} &<1. \label{eq:Proof-NewSSC-ss}
\end{align} This is a sufficient condition for every optimal solution of the Lasso SSC problem \eqref{eq:SSC-Lasso} to be subspace preserving. Solving with respect to $\lambda$, \eqref{eq:Proof-NewSSC-ss} gives
\begin{align}
\lambda < \frac{1}{2} \Bigg(\frac{1}{\zeta} + \frac{1}{\mu_{\lambda}} \Bigg) \label{eq:NewSSC-lambda-UB}.
\end{align} On the other hand, requiring that $\lambda > 1/ \zeta$ guarantees that $\c=\0$ is not optimal for \eqref{eq:SSC-Lasso}, exactly as in the proof of Theorem
\ref{thm:SSCuncorruptedOld}. Then condition $\mu_{\lambda}<\zeta$ ensures that 
the interval $\Lambda_{\lambda}$ is non-empty.

\section{Conclusions}
We developed theoretical bounds on the tolerable percentage of missing entries for Sparse Subspace Clustering (SSC). Our analysis confirmed theoretically that a projection onto the observed pattern of the point being expressed leads to a higher tolerable level of missing entries than without this projection, which is a finding of general interest for self-expressive clustering and robust PCA methods.

\appendix
\section{Appendix}

In this appendix we collect some technical elements from convex geometry and high-dimensional probability theory, that are used extensively in the technical development of the paper. Proposition \ref{prp:r-charecterization} is known to experts, yet being unaware of a direct proof in the literature, we provide one.

\begin{dfn} \label{dfn:InradiusCircumradiusPolar}[Inradius, Circumradius, Polar Set]
Let $\bY \in \Re^{D \times L}$ be a set of $L$ points of $\Re^{D}$. We denote by $r(\bY)$ the radius of the largest Euclidean ball of $\Span(\bY)$ that is contained in the convex hull of the points $\bY$; we refer to $r(\bY)$ as the relelative inradius. Dually, we denote by $\mathcal{R}(\bY)$ the radius of the smallest Euclidean ball of $\Span(\bY)$ that contains $\bY$; we refer to $\mathcal{R}(\bY)$ as the relative circumradius. The polar set $\bY^\circ$ of $\bY$ is the convex subset of $\Re^D$
\begin{align}
\bY^\circ:= \Big\{\v \in \Re^D: \, \, \, \y^\transpose \v \le 1, \, \forall \y \in \bY \Big\}.
\end{align}
\end{dfn}

\begin{lem}[\S $1.2$ in \cite{Gritzmann:92}] \label{lem:r-R}
Let $\bY \in \Re^{D \times L}$ be a set of $L$ points of $\Re^{D}$. Then 
\begin{align}
\mathcal{R}(( \bY)^\circ \cap \Span(\bY)) = 1 / r( \bY).
\end{align} 
\end{lem} 

\begin{prp} \label{prp:r-charecterization}
For $\bY \in \Re^{D \times L}$ a set of $L$ points of $\Re^{D}$,
\begin{align}
r(\bY) = \min_{\v: \norm{\v}_2=1} \norm{\bY^\transpose \v}_{\infty}.
\end{align}
\end{prp}
\begin{proof}
 From Lemma \ref{lem:r-R} we have that
\begin{align}
r(\bY) &= 1/\mathcal{R}\Big(( \bY)^\circ \cap \Span(\bY) \Big) \\
&=   1/\max_{\v \in (\pm \bY)^\circ \cap \Span(\bY)} \norm{\v}_2\\
&  = 1/\max_{\v \in \Span(\bY): \norm{\bY^\transpose \v}_{\infty}\le1} \norm{\v}_2  \\
&= 1/\max_{\v \in \Span(\bY): \norm{\bY^\transpose \v}_{\infty}=1} \norm{\v}_2\\
& =\min_{\v: \norm{\v}_2=1} \norm{\bY^\transpose \v}_{\infty},
\end{align} where in the last equality we used that
\small
\begin{align}
\sup_{g(v)=1} f(v ) = \sup_{g\big(\frac{v}{f(v)}\big)=\frac{1}{f(v)}} f(v) 
= 1/\inf_{f(u)=1} g(u),
\end{align} \normalsize for any $1$-homogeneous functions $f,g : \Re^D \rightarrow \Re$.
\end{proof}

\begin{lem}[\cite{Alonso:AMS2008,Soltanolkotabi:AS12}]\label{lem:r}
Let $\bY$ be a set of $L$ random points uniformly distributed on $\mathbb{S}^{d-1}$. Then for $\rho:=L/d$ larger than a constant, we have 
\begin{align}
\mathbb{P}\Bigg[r(\bY) \le \sqrt{\frac{\log(\rho)}{16 d}}\Bigg] \le \exp{(-\sqrt{\rho} d)}. 
\end{align}
\end{lem}

\begin{lem}[Follows from Theorem $1.5$ in \cite{Vershynin:GFA2009}]\label{lem:inner}
Let $\x, \hat{\v}$ be independent random vectors on $\mathbb{S}^{D-1}$, with $\x$ uniformly distributed. Then for every $\epsilon >0$ we have 
\begin{align}
\mathbb{P}\Big[|\x^\transpose \hat{\v}| \ge \epsilon \Big] \le 2\exp{(-D \epsilon^2 / 2)}.
\end{align}
\end{lem}

\begin{lem}[Follows from Lemma $5.3.2$ in \cite{Vershynin:HDP2017}] \label{lem:RandomProjection}
Let $\x$ be uniformly distributed on the unit sphere $\mathbb{S}^{D-1}$ and let $\V$ be a fixed linear subspace of $\mathbb{R}^D$ of dimension $d$. Then, for any $\epsilon>0$, the orthogonal projection $\bP_{\V}\x$ of $\x$ onto $\V$ satisfies
\begin{align}
&\mathbb{P}\Big[\sqrt{d/D}-\epsilon \le \norm{\bP_{\V}\x}_2 \le  \sqrt{d/D}+\epsilon \Big] \nonumber\\
& \ge 1 - 2 \exp\Big(-\frac{D\epsilon^2}{2} \Big). 
\end{align}
\end{lem} 

\section*{Acknowledgment}
The first author thanks Dr. Chun-Guang Li of the Beijing University of Posts and Telecommunications for very useful comments on an earlier version of this manuscript.

\bibliographystyle{plain}
\bibliography{biblio/alias,biblio/vidal,biblio/vidal-books,biblio/vision,biblio/math,biblio/math-Manolis,biblio/learning,biblio/sparse,biblio/geometry,biblio/dti,biblio/recognition,biblio/surgery,biblio/coding,biblio/matrixcompletion,biblio/segmentation,biblio/dataset}

\end{document}